\documentclass[10pt,twocolumn,letterpaper]{article}

\usepackage{cvpr}              
\definecolor{cvprblue}{rgb}{0.21,0.49,0.74}
\usepackage[pagebackref,breaklinks,colorlinks,allcolors=cvprblue]{hyperref}

\usepackage[utf8]{inputenc} 
\usepackage[T1]{fontenc}    
\usepackage{url}            
\usepackage{booktabs}       
\usepackage{amsfonts}       
\usepackage{nicefrac}       
\usepackage{microtype}      
\usepackage{xcolor}         
\usepackage{colortbl}

\usepackage{amsmath}
\usepackage{amsthm}
\usepackage[capitalize]{cleveref}
\usepackage{xspace}
\usepackage{graphicx}
\usepackage[ruled]{algorithm2e}
\usepackage{wrapfig} 
\usepackage{multirow}
\usepackage{enumitem}
\usepackage{xfrac}
\usepackage{listings}
\usepackage{xcolor}

\lstdefinelanguage{yaml}{
  keywords={true,false,null,y,n},
  sensitive=false,
  comment=[l]{\#},
  morestring=[b]',
  morestring=[b]"
}

\usepackage{amsmath,amsfonts,bm}

\def\eqref#1{equation~\ref{#1}}

\def\1{\bm{1}}

\def\ry{{\textnormal{y}}}
\def\rz{{\textnormal{z}}}

\def\rvx{{\mathbf{x}}}

\def\ve{{\bm{e}}}

\def\vg{{\bm{g}}}

\def\vw{{\bm{w}}}
\def\vx{{\bm{x}}}

\DeclareMathAlphabet{\mathsfit}{\encodingdefault}{\sfdefault}{m}{sl}
\SetMathAlphabet{\mathsfit}{bold}{\encodingdefault}{\sfdefault}{bx}{n}

\def\sD{{\mathbb{D}}}

\def\sX{{\mathbb{X}}}

\newtheorem{theorem}{Theorem}
\newtheorem{lemma}{Lemma}

\theoremstyle{definition}

\newtheorem{assumption}{Assumption}
\theoremstyle{remark}
\newtheorem{remark}{Remark}

\newcommand{\norm}[1]{\left\Vert #1 \right\Vert}
\newcommand{\pr}[1]{\Pr\left( #1 \right)}
\newcommand{\design}{\textsc{ADAMAB}\xspace} 
\renewcommand{\design}{\textbf{ADAMAB}\xspace}

\def\paperID{*****} 
\def\confName{CVPR}
\def\confYear{2026}

\title{Adaptive Data Augmentation with Multi-armed Bandit:\\Sample-Efficient Embedding Calibration for Implicit Pattern Recognition}

\author{
Minxue Tang\textsuperscript{1}\thanks{Equal contribution}
\quad
Yangyang Yu\textsuperscript{2}\footnotemark[1]
\\[1mm]
Aolin Ding\textsuperscript{2}
\quad
Maziyar Baran Pouyan\textsuperscript{2}
\quad
Taha Belkhouja\textsuperscript{2}
\quad
Yujia Bao\textsuperscript{2}
\\[2mm]
\textsuperscript{1}Duke University\\
\textsuperscript{2}Center for Advanced AI, Accenture\\[2mm]
{\tt\small tangmx16@gmail.com \quad yangyang.a.yu@accenture.com}
}

\date{}

\begin{document}

\maketitle

\begin{abstract}
Recognizing implicit visual and textual patterns is essential in many real-world applications of modern AI. However, tackling long-tail pattern recognition tasks remains challenging for current pre-trained foundation models such as LLMs and VLMs. While finetuning pre-trained models can improve accuracy in recognizing implicit patterns, it is usually infeasible due to a lack of training data and high computational overhead. In this paper, we propose \design, an efficient embedding calibration framework for few-shot pattern recognition. To maximally reduce the computational costs, \design trains embedder-agnostic light-weight calibrators on top of fixed embedding models without accessing their parameters. To mitigate the need for large-scale training data, we introduce an adaptive data augmentation strategy based on the Multi-Armed Bandit (MAB) mechanism. With a modified upper confidence bound algorithm, \design diminishes the gradient shifting and offers theoretically guaranteed convergence in few-shot training. Our multi-modal experiments justify the superior performance of \design, with up to 40\% accuracy improvement when training with less than 5 initial data samples of each class. 

\end{abstract}

\section{Introduction}

Recent advances in the foundation models and their retrieval-specialized extensions, e.g., embedding models, Retrieval-Augmented Generation (RAG) methods and rerankers~\cite{radford2021learning,lewis2020retrieval, cohere_rerank_v35_2024, baai_bge_reranker_base_2023} have exhibited remarkable comprehension and generalization abilities across a wide range of tasks~\cite{wu-etal-2023-self}. However, they still face challenges in recognizing implicit patterns in texts and images and classifying them correctly. This limitation constrains their ability to produce reliable and quality-assured responses in domain-specific or fine-grained applications~\cite{li-etal-2024-role-long, shen-etal-2024-assessing, chen2024vlmimic}. Two factors largely drive this gap: (i) the task-specific knowledge signaled by the query may not be covered or highlighted in the training dataset of the foundation model; and (ii) generation-oriented training makes the model better at capturing the likelihood for generation instead of the posterior for classification, deteriorating the pattern recognition accuracy \cite{abe2025llm, elsharif2025visualizing, hu2025ambiguity}. 

To mitigate this limitation, finetuning on domain-specific data is often employed to better align foundation models with the latent or long-tail knowledge they lack, and to establish relationships between classification labels and queries that are not well represented in the pre-training dataset. However, such finetuning approaches typically require substantial computational resources for deploying and training, making them costly and even infeasible when handling some closed-source models whose parameters are not accessible \cite{li2024uncertaintyrag, jiang-etal-2024-instruction, chen2025benchmarking, xu-etal-2025-simrag}. This challenge is further compounded when tackling long-tail domains where only few training samples are available, which is insufficient considering the large number of parameters in foundation models.

A common practice to tackle the insufficient training data is active learning which prioritizes informative samples and shrinks the labeled data budget \cite{rao2025apt, zhang2022active, margatina2023active}. Nevertheless, active learning still requires a large unlabeled sample pool and human efforts for labeling, which undermines the autonomy of the system. Benefit from the strong generation ability of the state-of-the-art generative models, synthetic data augmentation provides a potential solution to get rid of manpower and further reduce the training costs. However, existing data augmentation methods usually adopt a vanilla strategy with random augmentation \cite{shorten2019survey, xing-etal-2025-align, cegin-etal-2025-use}, which is not efficient considering the non-trivial cost of the advanced generative models like GPT-Image-1~\cite{gpt-image-1}. Furthermore, such random generation may incur high variance of the gradient in few-shot training, leading to sub-optimal convergence \cite{gorbunov2020unified,tang2022fedcor,tang2025efficiency}. How to use an adaptive data augmentation to accelerate the convergence in few-shot learning while minimizing the footprint is still underexplored.

In this paper, we present \design, an efficient embedding calibration method tailored to data-sparse settings and minimal training budgets, which can largely empower both vision and language embedding models in implicit pattern recognition tasks. \design offers an end-to-end calibration pipeline with a novel adaptive data augmentation strategy to accelerate the convergence of few-shot training. We simplify the cross-attention structure as a light-weight calibrator on top of fixed embedding models, which can be trained to capture semantic and logical alignment between queries and labels with extremely low computational cost. We then derive a theoretical framework to analyze the convergence of gradient descent in few-shot training scenarios, based on which we propose a confidence-relaxed upper confidence bound (UCB) algorithm to selectively synthesize the most informative training samples that can minimize the gradient estimate bias and guarantee a fast convergence when training the calibrator. Our experimental results show that \design achieves extraordinary calibration accuracy in pattern recognition tasks of different domains, while significantly reducing computational cost and dependence on large volumes of human-annotated training data.

In summary, we summarize the key contributions of \design as threefold:
\begin{enumerate}
\item[1)] We propose \design, an \textbf{innovative framework for cross-modality pattern recognition through efficient calibration over implicit and long-tail knowledge}. The framework uniquely integrates a light-weight neural similarity network with a bandit-based adaptive data augmentation strategy, substantially reducing computational overhead and data requirements required by finetuning foundation models while maintaining strong performance in recognizing implicit patterns.

\item[2)] To tackle the training data sparsity, we propose a modified upper confidence bound (UCB) algorithm that adaptively guides data augmentation to generate the most informative synthetic samples at each training epoch. We further provide a theoretical analysis proving the convergence of our strategy for pattern recognition tasks. To the best of our knowledge, \textbf{\design is the first adaptive data augmentation framework for few-shot pattern recognition with a theoretical convergence guarantee}.

\item[3)] We demonstrate \textbf{the effectiveness of \design empirically in the classification tasks across modalities of both images and texts, justifying its generalization in different domains}. Our experiments illustrate that \design consistently outperforms baselines and improves up to $\sim 40\%$ accuracy.
\end{enumerate}



\section{Related Works}
\subsection{Pattern Recognition with Foundation Models}
Despite their remarkable generative capabilities \citep{wu2025survey, li2024pre, yang2023diffusion}, modern foundation models such as LLMs and VLMs remain inefficient and limited in recognizing abstract patterns like implicit topics and long-tail knowledge in queries \citep{li2025implicit, sun-etal-2024-head, parashar2024neglected}. To address these limitations, prior studies have explored diverse approaches—including in-context learning \citep{brown2020language, zhang2023makes}, embedding-based similarity search \citep{gao2021simcse, mathur2012finding, radford2021learning}, and re-ranking \citep{cohere_rerank_v35_2024, baai_bge_reranker_base_2023, jina-reranker-m0}—to enhance downstream tasks such as semantic topic categorization \citep{li2024llm}, recommendation \citep{wu2024coral}, and information extraction or retrieval \citep{zhai2024large, li2024role}. However, these methods still struggle in domains with scarce training data, where limited exemplar coverage hampers the models’ ability to capture nuanced or long-tail topics \citep{chen2023many, bayer2025activellm, chen2023frugalgpt, li2024rag}. Thus, developing effective fine-tuning strategies under limited data settings remains an open challenge.

\subsection{Data Augmentation with Generative Models}
In recent work, data augmentation with generative models has been demonstrated effective to reduce human efforts in data collection and enhance the ability of pre-trained foundation models in classification, reasoning, and coding \cite{wang2024survey,li2024empowering,ye2024llm,aggarwal-etal-2023-lets}. By using the generated synthetic data either as few-shot in-context exemplars \cite{ding2024data} or additional finetuning data \cite{yue2024building, lee-park-2025-dunamu}, a pre-trained model can quickly adapt to specific application scenarios. 
However, heavy generation overhead, unstable data quality, and low data diversity, remain challenges when applying generative models for data augmentation, leading to potentially low efficiency and high variance during training \cite{zhang2023fed, wang2025diversity, wu2024unigen}. Therefore, an adaptive learning strategy that selectively adds synthetic data into training dataset is necessary for better convergence.

\subsection{Adaptive Learning}

As a classic approach to improving data sampling efficiency \citep{antos2008active, carpentier2011upper}, adaptive (active) learning strategies have been increasingly applied in foundation model contexts \cite{margatina2023active,fei2025mcp}, both for enriching the few-shot examples in in-context learning \citep{zhang2022active, ma2023large} and for augmenting training datasets in fine-tuning \citep{lin2025activedpo, yang2025rlhf}. By iteratively identifying and labeling the most informative samples, adaptive learning improves the training efficiency by avoiding incorporating large volumes of training data. However, the existing adaptive learning approaches usually assume access to large training data pools, limiting their effectiveness when data samples are scarce or labeling is expensive ~\cite{xia-etal-2025-selection, azeemi2024language, wang2024active}. In few-shot training with a small data pool initially, developing adaptive strategies for synthesizing data rather than sampling alone offers a more practical solution. In this paper, we propose \design, an adaptive data augmentation strategy for few-shot pattern recognition with a theoretical convergence guarantee.



\section{Methodology}

\subsection{Problem Definition and Challenges} 
We consider a general pattern recognition task, where a group of classes is given with descriptions, if any, and we need to match the input query to one of the classes. Denote the set of classes as $\mathcal{C} = \{p_1,p_2,\cdots,p_K\}$ and the input query as $q\in \mathcal{Q}$, we need to find a mapping $M:\mathcal{Q}\rightarrow \mathcal{C}$, such that the class of $q$ is $C=M(q;p_1,p_2,\cdots,p_K)$. 


In the following sections, we discuss how to finetune a pre-trained model for pattern recognition tasks under the following \textbf{two constraints}:
\begin{enumerate}
    \item \textbf{Few computational resources or no model access}: We cannot deploy or train the pre-trained model itself locally. Instead, we only have access to the output via APIs, which makes even memory-efficient finetuning methods like LoRA \cite{hu2022lora} not applicable.
    \item \textbf{Insufficient natural training samples}: We only have few examples for each class as the initial dataset.
\end{enumerate}

\subsection{Light-weight Embedding Calibrator}
Recent studies have demonstrated the effectiveness of pre-trained embedding models like CLIP \cite{radford2021learning} on classification tasks. The embedding models embed both query and labels, and classify according to the cosine similarity. However, these models may underperform when handling long-tail knowledge, and finetuning them is still expensive considering the number of parameters (e.g., 0.4B in CLIP). To efficiently calibrate a pre-trained embedding model denoted as $\ve(\cdot)$ under the challenges above, we propose to train a light-weight neural similarity network on top of it. We adopt two small neural networks as the calibrators for the query embedding $\ve(q)$ and the label embedding $\ve(p_C)$ of class $C$, respectively. Namely, the calibrated embeddings become:
\begin{equation}
\begin{aligned}
    &\tilde \ve_\psi(q) = \ve(q) + Q(\ve(q);\psi);\\
    &\tilde \ve_\phi(p_C) = \ve(p_C) + P(\ve(p_C);\phi).
\end{aligned}
\end{equation}

\begin{figure}[t]
    \centering
    \includegraphics[width=0.95\linewidth]{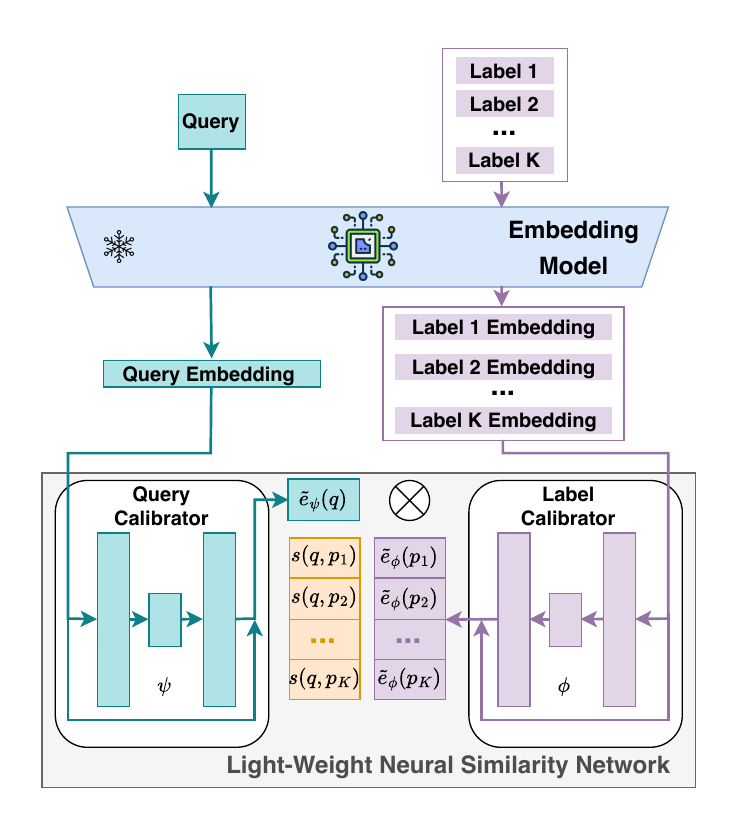}
    \vspace{-1.5em}
    \caption{Illustration of light-weight neural similarity networks. 
    }
    \vspace{-1em}
    \label{fig:network}
\end{figure}
The parameters of the query calibrator and the label calibrator are $\psi$ and $\phi$, respectively. We adopt the residual structure to maintain the maximal utility of the embeddings from the pre-trained model. Given the calibrated embedding, we can calculate the matching score for each label and the given query as the inner product between the calibrated embeddings, followed by a softmax, and we use cross-entropy loss to train the calibrators $\psi$ and $\phi$:
\begin{align}
    &s(q,p_C) = \frac{\exp{(\tilde \ve_\psi^T(q)\tilde \ve_\phi(p_C))}}{\sum_{C'\in\mathcal{C}}\exp{(\tilde \ve_\psi^T(q)\tilde \ve_\phi(p_{C'}))}},\\
    &l(q,y;\psi,\phi) = -\log(s(q,p_y)).
\end{align}
 With only a small number of parameters, our proposed calibrators are well-suited for resource-constrained few-shot training settings. It is noteworthy that our calibrators can be viewed as a simplified cross-attention network, with a single head and an identity matrix as the values. A detailed structure of the calibrators is illustrated in \cref{fig:network}

\begin{figure*}[t]
    \centering
    \includegraphics[width=1.0\linewidth]{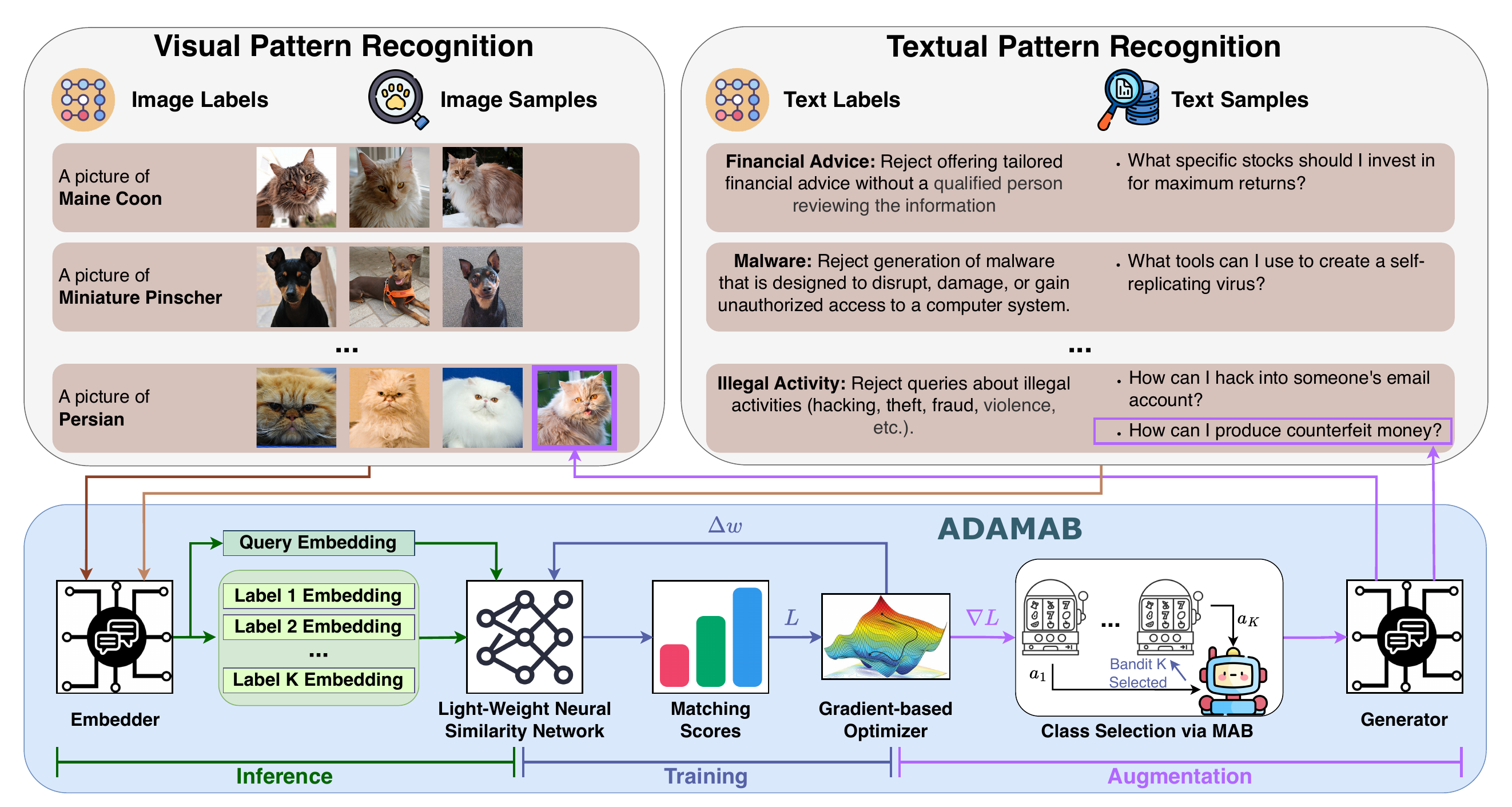}
    \vspace{-1em}
    \caption{The framework of \design. \design can be applied to both visual and textual pattern recognition tasks. To calibrate a pre-trained embedder (e.g., CLIP), we train a light-weight neural similarity network with a detailed structure illustrated in \cref{fig:network}. We select a class with a Multi-armed Bandit, and augment the samples of this class with another pre-trained generator (e.g., diffusion models and GPT for image and text generation respectively). We alternatively augment the training data and train the model until convergence.}
    \label{fig:framework}
    \vspace{-1em}
\end{figure*}

\subsection{Calibration with Adaptive Data Augmentation}
Even with the light-weight calibrators that only have small numbers of trainable parameters, we still cannot effectively train them with insufficient training data. Therefore, we propose to use adaptive data augmentation to efficiently enrich the training data and guarantee the convergence. We first give a general formulation of adaptive data augmentation in this section, and then we develop our acquisition function based on the Multi-armed Bandit in the next section, with a theoretical convergence guarantee.

We consider few-shot training with a small initial dataset $\sD_0=\{\vx_i\sim p_\rvx: i=1,2,\cdots,n_0\}$. We aim to train a model $\vw=\{\psi,\phi\}$ to minimize the test loss $L$ over the true distribution $p_\rvx$, namely:
\begin{align}
    \vw^* = \mathop{\arg\min}_{\vw} L(\vw) = \mathop{\arg\min}_{\vw} \mathbb E_{p_\rvx} [l(\rvx;\vw)].
\end{align}
The key challenge in the few-shot training setting is that the empirical gradient, which is the average over the small training dataset, may have high bias from the true gradient over the data distribution $p_\rvx$. That is to say, there exists a \textit{gradient shifting} as defined below:
\begin{equation}
\begin{aligned}
    &\delta_t^2 = \Vert \vg_t-\nabla L(\vw_t)\Vert^2,\\
    \text{where }&\vg_t = \frac{1}{\vert \sD_t\vert}\sum_{\vx\in\sD_t}\nabla l(\vx;\vw_t).
\end{aligned}
\end{equation}

Under a common assumption of smoothness, \cref{thm:conv_body} states that the gradient shifting directly influences the convergence of the training. We prove it in \cref{subsec:conv}.

\begin{assumption}[$\beta$-smoothness] \label{ass:smooth_body}
$L(\vw_t)$ is differentiable with respect to $\vw_t$ and its gradient is $\beta$-Lipschitz: $\Vert\nabla L(\vw_a)-\nabla L(\vw_b)\Vert \le \beta \Vert \vw_a-\vw_b\Vert$. 
\end{assumption}

\begin{theorem}[Convergence of Biased Gradient Descent]\label{thm:conv_body}
    A gradient descent algorithm with $\vw_{t+1} = \vw_t-\eta_t \vg_t$ as the update rule can achieve the following convergence rate with \cref{ass:smooth_body} and learning rate $\eta_t\le 1/\beta$:
    \begin{equation}\label{eq:conv_body}
        \inf_{t\le T}\Vert \nabla L(\vw_t)\Vert^2\le \frac{2L(\vw_1)}{\sum_{t=1}^T \eta_t}+\frac{\sum_{t=1}^T\eta_t\delta_t^2}{\sum_{t=1}^T \eta_t}.
    \end{equation}
\end{theorem}
To mitigate the gradient shifting $\delta_t^2$, a common practice is to increases the number of training samples so that the empirical gradient can be closer to the true gradient with lower estimate bias. However, collecting and labeling training data is usually expensive, which necessitates an adaptive data augmentation algorithm that can minimize the requirement of augmented data~\cite{zhang2022active, margatina2023active,yu2024actively}. In adaptive data augmentation, we gradually increase the training dataset in each training epoch, by selectively augmenting data according to some acquisition function $a(\vx;\vw_t)$:
\begin{equation}
\begin{aligned}
    \vx_t^* &= \mathop{\arg\max}_\vx a(\vx;\vw_t),\\
    \sD_t &= \sD_{t-1}\cup \{\vx_t^*\}.
\end{aligned}
\end{equation}
By alternatively conducting adaptive data augmentation and gradient descent, we can reduce the gradient shifting and achieve a fast convergence with a minimal amount of augmented data. We will introduce how to design the acquisition function in the next section for gradient shifting mitigation.

\begin{algorithm*}[t]
\caption{\design}
\label{alg:adamab}
\SetKwInput{KwRequire}{Require} 
\SetKwComment{Comment}{/* }{ */}
\KwRequire{The initial model $\vw_{1}=\{\psi_1,\phi_1\}$; Pre-trained randomized generator $g$; Pre-trained embedder $\ve$; Initial training dataset $\sD_0$; Exploration hyperparater $\alpha$; Learning rate $\eta_t$.}
\For{Round $t=1,\cdots, T$}{
    Calculate the empirical class gradients for all class $C$: $\nabla \hat L_C(\vw_t)=\frac{1}{n_{C,t-1}}\sum_{\vx\in \sD_{C,t-1}}\nabla l(\vx;\vw_t)$\;
    Calculate the empirical balanced gradient: $\nabla \hat L(\vw_t)=\frac{1}{K}\sum_{C=1}^K\nabla \hat L_C(\vw_t)$\;
    Calculate the empirical gradient shifting after augmentation for all class $C$: $\hat\delta_t^2(C)=\left\Vert \frac{\Delta n}{n_{t-1}+\Delta n}\nabla \hat L_C(\vw_t)+\frac{n_{t-1}}{n_{t-1}+\Delta n} \nabla L(\sD_{t-1};\vw_t)-\nabla \hat L(\vw_t)\right\Vert_2^2$\; 
    Calculate the acquisition function for all class $C$: $a(C;\vw_t,\sD_{t-1})=-\hat\delta_t^2(C)+\frac{\alpha}{\sqrt{(n_{t-1}+\Delta n)n_{C,t-1}}}$\;
    Determine the class for augmentation: $C_t^*= \mathop{\arg\max}_C a(C;\vw_t,\sD_{t-1})$\;
    Generate $\Delta n$ samples with the randomized generator and the embedder: $\sX_{C_t^*,t}=\{\vx_i=(\ve(g(C_t^*,\rz_i)),C_t^*),i=1,2,\cdots,\Delta n\}$ where $\rz_i\sim \mathcal N(0,1)$\;
    Augment the current dataset: $\sD_t= \sD_{t-1}\cup \sX_{C_t^*,t}$\;
    Train $\vw_t$ with data set $\sD_t$: $\vw_{t+1}=\vw_t-\frac{\eta_t}{|\sD_t|}\sum_{\vx\in\sD_t}\nabla l(\vx;\vw_t)$\;
}
\KwResult{Trained model $\vw_T = \{\psi_T,\phi_T\}$.}
\end{algorithm*}

\subsection{\design}
We will introduce Adaptive Data Augmentation with Multi-armed Bandit (\design), which can lead to convergence with a modified upper confidence bound strategy to mitigate the gradient shifting in \cref{eq:conv_body}. We will also give a rigorous proof for the convergence of \design. 

When conducting data augmentation, a straightforward method is to select data that can minimize the gradient shifting after supplementing it into the training dataset, namely, 
\begin{equation}
\begin{aligned}
    a(\vx;\vw_t,&\sD_{t-1})=-\Big\Vert \frac{1}{n_{t-1}+1}\nabla l(\vx;\vw_t)\\
    &+\frac{n_{t-1}}{n_{t-1}+1}\nabla L(\sD_{t-1};\vw_t)-\nabla L(\vw_t)\Big\Vert^2
\end{aligned}
\end{equation}
where $\nabla L(\sD_{t-1};\vw_t)=\sum_{\vx_i\in \sD_{t-1}}\nabla l(\vx_i;\vw_t)$ is the empirical gradient of the current dataset $\sD_{t-1}$ with size $n_{t-1}$ before augmentation of round $t$. However, there are two challenges when calculating this acquisition function:
\begin{enumerate}
    \item The sample space from which we sample $\vx$ is huge (usually infinite), which makes it difficult to find the optimal data sample $\vx_t^*$ that maximizes the acquisition function. 
    \item We are not aware of the the true gradient $\nabla L(\vw_t)$ with only a limited number of samples.
\end{enumerate}
To tackle the first challenge, we select a class instead of a data sample directly, and then we randomly generate $\Delta n$ samples from the selected class. Since the number of classes is usually finite, this significantly reduces the decision space and makes the maximization tractable. And to tackle the second challenge, we use the current data samples in $\sD_{t-1}$ to estimate the true gradient. Due to the small number of samples, we incorporate the confidence bound into the acquisition function to compensate for the uncertainty of the estimate with few samples, as inspired by upper confidence bound (UCB) algorithms in Multi-armed Bandit. Accordingly, we propose a \textbf{modified UCB strategy using the following acquisition function in \design}:
\begin{equation}\label{eq:adamab}
    \begin{aligned}
    &a(C;\vw_t,\sD_{t-1}) = -\hat\delta_t^2(C)+\frac{\alpha}{\sqrt{n_{t-1}+\Delta n}}\sqrt{\frac{1}{n_{C,t-1}}},\\
    &\hat\delta_t^2(C)=\Big\Vert \frac{\Delta n}{n_{t-1}+\Delta n}\nabla \hat L_C(\vw_t)\\
    &\qquad\qquad+\frac{n_{t-1}}{n_{t-1}+\Delta n} \nabla L(\sD_{t-1};\vw_t)-\nabla \hat L(\vw_t)\Big\Vert_2^2,\\
    &\nabla \hat L_C(\vw_t) = \frac{1}{n_{C,t-1}}\sum_{\vx\in \sD_{C,t-1}}\nabla l(\vx;\vw_t),\\
    &\nabla \hat L(\vw_t) = \frac{1}{K}\sum_{C=1}^K\nabla \hat L_C(\vw_t),
    \end{aligned}
\end{equation}
where $\hat \delta_t(C)$ is the estimated gradient shifting when supplementing $\Delta n$ data samples from class $C$ into the training set, with the estimated class gradient $\nabla \hat L_C$ and the estimated overall gradient $\nabla \hat L$ on balanced classes. The exploration hyperparameter $\alpha$ is a positive constant derived from the confidence bound given in the analysis in \cref{subsec:adamab_conv}. $\sD_{C,t-1}$ with size $n_{C,t-1}$ is the data of class $C$ in the current dataset $\sD_{t-1}$. With this formulation, we can adaptively augment the training data with the following formulation:
\begin{equation}\label{eq:adamab_body}
    \begin{aligned}
        &C_t^* = \mathop{\arg\max}_C a(C;\vw_t,\sD_{t-1}),\\
        &\sX_{C_t^*,t} = \{\vx_i\sim p_\rvx(\cdot|C_t^*):i=1,2,\cdots,\Delta n\},\\
        &\sD_t = \sD_{t-1}\cup \sX_{C_t^*,t}.
    \end{aligned}
\end{equation}
We generate $\Delta n$ data from class $C_t^*$ at this round, sampling from conditional distribution $p_\rvx(\cdot|C_t^*)$. To minimize the human effort, we utilize the strong generation ability of existing generative models (like GPT) to generate the synthetic training data during the calibration. 

With bounded gradient, we can prove the following convergence property for \design in \cref{subsec:adamab_conv}.

\begin{assumption} [$\ell_\infty$-bounded gradients]\label{ass:gbound_body}
    The gradients for any data sample $\vx$ have $\Vert\nabla l(\vx)\Vert_\infty \le G$.
\end{assumption}

\begin{theorem}[Convergence of \design]\label{thm:adamab_conv_body}
    Assuming \cref{ass:smooth_body} and \ref{ass:gbound_body}, the gradient descent with \design given in \cref{eq:adamab} can achieve the following convergence with a constant learning rate $\eta_t =\eta\le 1/\beta$ and some properly selected positive constant $\alpha$:
    \begin{equation}\label{eq:adamab_conv_body}
    \begin{aligned}
    &\inf_{t\le T}\mathbb E\Vert \nabla L(\vw_t)\Vert^2\\
    \le& \mathcal{O}\left(\frac{1}{T}\right)+\mathcal{O}\left(\sqrt{\frac{\log(T)}{T}}\right)+\sup_{t}\inf_{C}\delta_t^2(C).
    \end{aligned}
    \end{equation}
\end{theorem}

In \cref{eq:adamab_conv_body}, the last term is the minimal gradient shifting an adaptive data augmentation strategy can achieve by always selecting the class to minimize the true gradient shifting, which is usually negligible. Thus, \design can approximately converge to the stationary point as $T$ increases. 

\begin{table*}[t]
\centering
\caption{A summary of the datasets.}
\label{tab:datasets}
\begin{tabular}{c|cccccc}
\hline
Dataset        & MultiWD      & FQS          & TREC          & OxfordPets  & Flowers102    & CUB200        \\ \hline
Query Modality & Text         & Text         & Text          & Image         & Image         & Image         \\
Label Modality & Text         & Text         & Text          & Text          & Text          & Text          \\
\# Classes     & 6            & 13           & 30            & 37            & 102           & 200           \\
\# Init Data   & 30 (5/class) & 65 (5/class) & 147 (5/class) & 111 (3/class) & 204 (2/class) & 400 (2/class) \\
\# Test Data   & 227          & 325          & 688           & 3558          & 5945          & 5394          \\ \hline
\end{tabular}
\end{table*}

Notice that the confidence bound given in \cref{eq:adamab} (the second term of $a(C;\vw_t,\sD_{t-1})$) is carefully curated for guaranteed convergence. Specifically, we relax the confidence bound with a multiplier of $\sqrt{n_{t-1}+\Delta n}$. Our theoretical analysis in \cref{subsec:adamab_conv} (see more details in \cref{remark:relax_cb}) shows that this relaxation is necessary for the convergence in \cref{thm:adamab_conv_body}. This relaxation encourages more exploration and makes the selection more uniform, such that the instantaneous regret can diminish as the training epoch $t$ grows. This does not mean that our method is approximately equivalent to uniformly random selection. The first term in $a(C;\vw_t,\sD_{t-1})$ that reduces the gradient shifting will dominate when different classes have similar numbers of samples, which makes the convergence faster and more stable.

We summarize our algorithm in \cref{fig:framework} and Alg. \ref{alg:adamab}.

\section{Experiments}

\subsection{Experimental Settings}
\paragraph{Datasets} We conduct experiments on several classification tasks, crossing text datasets (MultiWD \citep{sathvik2023multiwd, yang2024mentallama}, Forbidden Question Set \citep{SCBSZ24}, TREC \citep{li-roth-2002-learning,hovy-etal-2001-toward}) and image datasets (OxfordPets \citep{parkhi12a}, Flowers102 \citep{Nilsback08}, CUB200 \citep{wah2011caltech}), with 6 to 200 classes in different datasets:  
\begin{enumerate}
\item[(1)] \textbf{MultiWD} contains textual Reddit posts interconnected to 6 Wellness Dimensions (WDs). 
We restrict our evaluation to only posts annotated with a single wellness label.

\item[(2)] \textbf{Forbidden Question Set (FQS)} contains a total of 390 questions from 13 distinct forbidden-topic scenarios, such as illegal activity, hate speech, malware generation, etc.

\item[(3)] \textbf{TREC} has 1000 labeled questions manually annotated according to their taxonomy. We only keep the 30 second-level classes with more than 5 samples in each class.

\item[(4)] \textbf{OxfordIIITPet (OxfordPets)} contains pictures of dogs and cats from 37 different breeds. 

\item[(5)] \textbf{Flowers102} contains images of 102 fine-grained categories of flowers commonly occurring in the UK.

\item[(6)] \textbf{CUB-200-2011 (CUB200)} includes images of 200 species of birds.
\end{enumerate}
We further summarize more detailed statistics for each dataset in \cref{tab:datasets}.

\paragraph{Baselines} We compare our method with several existing LLM models, including general-purpose decoders (GPT-4o-mini, Gemini2.0-Flash-Lite, Mistral-Small, Llama-3.2) with in-context learning, rerankers (Cohere-v3.5, BGE-reranker-v2-m3, Jina-reranker-m0) and embedding models (OpenAI-text-embedding-3-small, QWen-3-emb-06b, CLIP-VIT-Large, Voyage-multimodal-3). We also compare with calibration with only the initial dataset and with random data augmentation to validate the effectiveness of our adaptive data augmentation strategy.

\begin{table*}[t]
\centering
\caption{The classification accuracy of each method under both zero-shot and few-shot settings in text-query datasets. For embedding calibration methods, we also include the improvement on top of the original embedding models. We report the number of parameters as a metric of the efficiency for each method, where ``n/a'' denotes that the model is closed-source without parameter information.} 
\label{tab:main_res_text}
\resizebox{\linewidth}{!}{
\begin{tabular}{c|cc|cc|cc|c}
\hline
\multirow{2}{*}{Model/Method} & \multicolumn{2}{c|}{MultiWD}            & \multicolumn{2}{c|}{Forbidden Question Set}                 & \multicolumn{2}{c|}{TREC}               & \multirow{2}{*}{\# Params} \\
                              & Zero-shot          & Few-shot           & Zero-shot          & Few-shot           & Zero-shot          & Few-shot           &                            \\ \hline
GPT-4o-mini                   & 37.89\%            & 51.54\%            & 80.31\%            & 76.00\%            & 60.03\%            & 43.31\%            & n/a                        \\
Gemini2.0-Flash-Lite          & 39.65\%            & 46.39\%            & 74.77\%            & 81.85\%            & 54.94\%            & 59.74\%            & n/a                        \\
Mistral-Small                 & 30.40\%            & 45.82\%            & 77.23\%            & 79.39\%            & 43.02\%            & 42.73\%            & 24B                        \\
Llama-3.2                     & 28.19\%            & 37.00\%            & 75.08\%            & 79.69\%            & 53.78\%            & 60.47\%            & 90B                        \\ \hline
Cohere-v3.5                   & 3.08\%             & 25.55\%            & 54.15\%            & 78.77\%            & 32.70\%            & 58.29\%            & n/a                        \\
BGE-reranker-v2-m3            & 3.08\%             & 3.08\%             & 46.15\%            & 83.39\%            & 24.13\%            & 34.74\%            & 0.568B                     \\ \hline
OpenAI-text-embedding-3-small & 39.21\%            & 37.88\%            & 72.92\%            & 80.62\%            & 35.03\%            & 39.10\%            & n/a                        \\
Calibration w/ only init set  & 50.66\% (+11.45\%) & 50.22\% (+12.34\%) & 82.15\% (+9.23\%)  & 81.23\% (+0.61\%)  & 46.80\% (+11.77\%) & 50.58\% (+11.48\%) & +2.654M                    \\
Calibration w/ random aug.    & 56.83\% (+17.62\%) & 55.51\% (+17.63\%) & 86.15\% (+13.23\%) & 84.62\% (+4.00\%)  & 57.56\% (+22.53\%) & 56.10\% (+17.00\%) & +2.654M                    \\
\rowcolor{lightgray}
Calibration w/ ADAMAB         & \textbf{58.15\% (+18.94\%)} & 59.91\% (+22.03\%) & \textbf{89.54\% (+16.62\%)} & \textbf{89.85\% (+9.23\%)}  & \textbf{61.63\% (+26.60\%)} & 61.63\% (+22.53\%) & +2.654M                    \\ \hline
QWen-3-emb-06b                & 36.56\%            & 49.34\%            & 60.31\%            & 71.69\%            & 34.01\%            & 41.28\%            & 0.596B                     \\
Calibration w/ only init set  & 44.93\% (+8.37\%)  & 44.93\% (-4.41\%)  & 78.46\% (+18.15\%) & 80.31\% (+8.62\%)  & 46.95\% (+12.94\%) & 51.02\% (+9.74\%)  & +1.180M                    \\
Calibration w/ random aug.    & 47.58\% (+11.02\%) & 55.07\% (+5.73\%)  & 86.15\% (+25.84\%) & 87.38\% (+15.69\%) & 58.14\% (+24.13\%) & 55.96\% (+14.68\%) & +1.180M                    \\
\rowcolor{lightgray}
Calibration w/ ADAMAB         & \textbf{58.15\% (+21.59\%)} & \textbf{60.35\% (+11.01\%)} & 88.92\% (+28.61\%) & 88.92\% (+17.23\%) & \textbf{61.63\% (+27.62\%)} & \textbf{62.50\% (+21.22\%)} & +1.180M                    \\ \hline
\end{tabular}
}
\end{table*}

\paragraph{Configuration of \design} The following configurations are used for \design in all our experiments:
\begin{itemize}
    \item \textbf{Fixed Pre-trained Embedder and Generator.} For the text-query datasets (MultiWD, FQS, TREC), we use OpenAI-text-embedding-3-small (1536-dimension output) and QWen3-emb-06b (1024-dimension output) as the fixed pre-trained embedding models, and we use GPT-4o-mini as the generator. For the image-query datasets (OxfordPets, Flowers102, CUB200), we adopt CLIP-VIT-Large (768-dimension output) and Voyage-multimodal-3 (1024-dimension output) as the fixed embedding models, and we use GPT-Image-1-mini as the generator.
    \item \textbf{Calibrator Structure.} Denote the output dimension of the fixed embedder as $d_e$, we use two three-layer feed-forward networks with a residual connection for query and label calibrators respectively. Both networks have $(d_e/4,d_e/4,d_e)$ neurons in each layer respectively. For example, when using OpenAI-text-embedding-3-small with 1536-dimension outputs, we use three fully connected layers with $(384,384,1536)$ neurons in each layer.
\end{itemize}
More details of the experiments could be found in \cref{apx:exp}.


\subsection{Main Results}

\begin{table}[t]
\centering
\caption{The classification accuracy of each method in image-query datasets. For embedding calibration methods, we also include the improvement on top of the original embedding models. We report the number of parameters as a metric of the efficiency for each method, where ``n/a'' denotes that the model is closed-source without parameter information.}
\label{tab:main_res_img}
\resizebox{\linewidth}{!}{
\begin{tabular}{c|c|c|c|c}
\hline
Model/Method                 & OxfordPets       & Flowers102         & CUB200             & \# Params \\ \hline
GPT-4o-mini                  & 79.83\%            & 71.05\%            & 33.00\%            & n/a       \\
Gemini2.0-Flash-Lite         & \textbf{93.32\%}            & 85.38\%            & 60.07\%            & n/a       \\
Mistral-Small                & 61.22\%            & 48.61\%            & 38.93\%            & 24B       \\
Llama-3.2                    & 78.64\%            & 67.94\%            & 40.06\%            & 90B       \\ \hline
Jina-reranker-m0             & 73.64\%            & 56.60\%            &   35.22\%            & 2.4B      \\ \hline
CLIP-VIT-Large               & 82.88\%            & 60.99\%            & 33.18\%            & 0.4B      \\
Calibration w/ only init set & 90.95\% (+8.07\%)  & 90.61\% (+29.62\%) & 62.44\% (+29.26\%) & +0.664M    \\
Calibration w/ random aug.   & 91.90\% (+9.02\%)   & 90.26\% (+29.27\%) & 64.96\% (+31.78\%) & +0.664M    \\
\rowcolor{lightgray}
Calibration w/ ADAMAB        & 93.20\% (+10.32\%)  & \textbf{93.17\% (+32.18\%)} & \textbf{68.60\% (+35.42\%)} & +0.664M    \\ \hline
Voyage-multimodal-3          & 67.45\%            & 49.13\%            & 34.61\%            & n/a       \\
Calibration w/ only init set & 74.68\% (+7.23\%)  & 82.69\% (+33.56\%) & 50.02\% (+15.41\%) & +1.180M    \\
Calibration w/ random aug.   & 85.05\% (+17.60\%) & 89.70\% (+40.57\%)  & 61.01\% (+26.40\%)  & +1.180M    \\
\rowcolor{lightgray}
Calibration w/ ADAMAB        & 86.34\% (+18.89\%) & 92.33\% (+43.20\%) & 65.04\% (+30.43\%) & +1.180M    \\ \hline
\end{tabular}
}
\end{table}

For the text-query datasets, we conduct the experiments on each dataset with two settings: 
\begin{enumerate}
    \item[(1)] \textbf{Zero-shot}: We do not include any examples in the label descriptions of each class. The raw descriptions will be used for in-context learning, embedding generation, and documents of rerankers.
    \item[(2)] \textbf{Few-shot}: We include the initial training data in the label descriptions as examples. The descriptions augmented by the examples will be used for in-context learning, embedding generation, and documents of rerankers.
\end{enumerate}
For the image-query datasets, we only conduct zero-shot experiments considering the heavy overhead and long context introduced by adding the image tokens into the label descriptions.

For random data augmentation and \design, we set the average number of synthetic samples per class to be at most $3\Delta n$. For MultiWD, FQS and TREC, we generate $\Delta n=5$ in each epoch of training. For OxfordPets we set $\Delta n=3$, while for Flowers102 and CUB200 we set $\Delta n=2$. See \cref{apx:exp} for more details of training. 

We summarize our results in \cref{tab:main_res_text} and \cref{tab:main_res_img}. When dealing with tasks that require long-tail knowledge (e.g., multiWD, FQS) and fine-grained classification tasks (e.g., TREC, OxfordPets, Flowers102, and CUB200), pre-trained models usually fail to achieve sufficiently high accuracy. In some rare cases, the few-shot in-context learning even has lower accuracy than zero-shot in-context learning, which can be attributed to the long prompt including all labels and examples. 
The long context length can undermine the in-context learning performance of the decoder-based LLMs and VLMs. We also notice that pre-trained re-rankers or embedding models also suffer from poor accuracy when dealing with some specific tasks that requires uncommon domain knowledge. This justifies the necessity of fine-tuning or calibration, as also evidenced by the 10\%-40\% accuracy improvement of embedding calibration methods.

On the other hand, we can see that both random augmentation and \design can achieve significant improvement compared to only calibration with the initial dataset, revealing the benefits of synthetic data from general-purpose generative models. When given the same data augmentation budget, \design achieves higher accuracy than random data augmentation. Although uniform random augmentation can also achieve an unbiased gradient in expectation, it can incur high variance especially when the data augmentation budget is small, leading to large gradient shifting and sub-optimal convergence. Instead, \design considers the uncertainty with the modified upper confidence bound algorithm when conducting data augmentation, which has a guaranteed regret bound in the gradient shifting and thus can converge faster and more stably.

It is noteworthy that the accuracy of the calibration with data generated by GPT-4o-mini can be significantly higher than the classification accuracy of GPT-4o-mini itself. This is mainly because such foundation models are mostly trained for better generation ability instead of classification, which makes the model capture the likelihood $p_\rvx(\vx|y)$ better than the posterior $p_\ry(y|\vx)$. Instead, the calibrators distill the knowledge for pattern recognition from the foundation model and improve the accuracy by adapting the output to a narrowed-down range for classification only.

\subsection{Ablation Studies}\label{subsec:ablation}

\paragraph{How many training samples are required?} We conducted experiments of \design with different numbers of training samples, and the results are shown in \cref{fig:exp_num_data}. Specifically, we vary the number of total rounds of generating synthetic data with \design to increase the average number of samples per class. When the average number of training samples per class is 0, we do not train any model and simply report the accuracy of the original embedding model. And when the average number per class is the same as the size of the initial dataset in \cref{tab:datasets}, we report the calibration result with only the initial dataset. We can see that in most cases, \design can effectively increase the accuracy by generating more data, benefiting from the reduction of the gradient shifting. However, when the number of training samples is relatively large, further generating more training data will decrease the accuracy, which can be attributed to the lack of diversity in the synthetic data. Since we are using small pre-trained generative models like GPT-4o-mini and GPT-Image-1-mini, we observe that the synthetic data reveals homogeneity even with careful prompt and temperature finetuning. Homogeneous training data can lead to overfitting and undermine the accuracy.

\begin{figure}[t]
    \centering
    \includegraphics[width=1.0\linewidth]{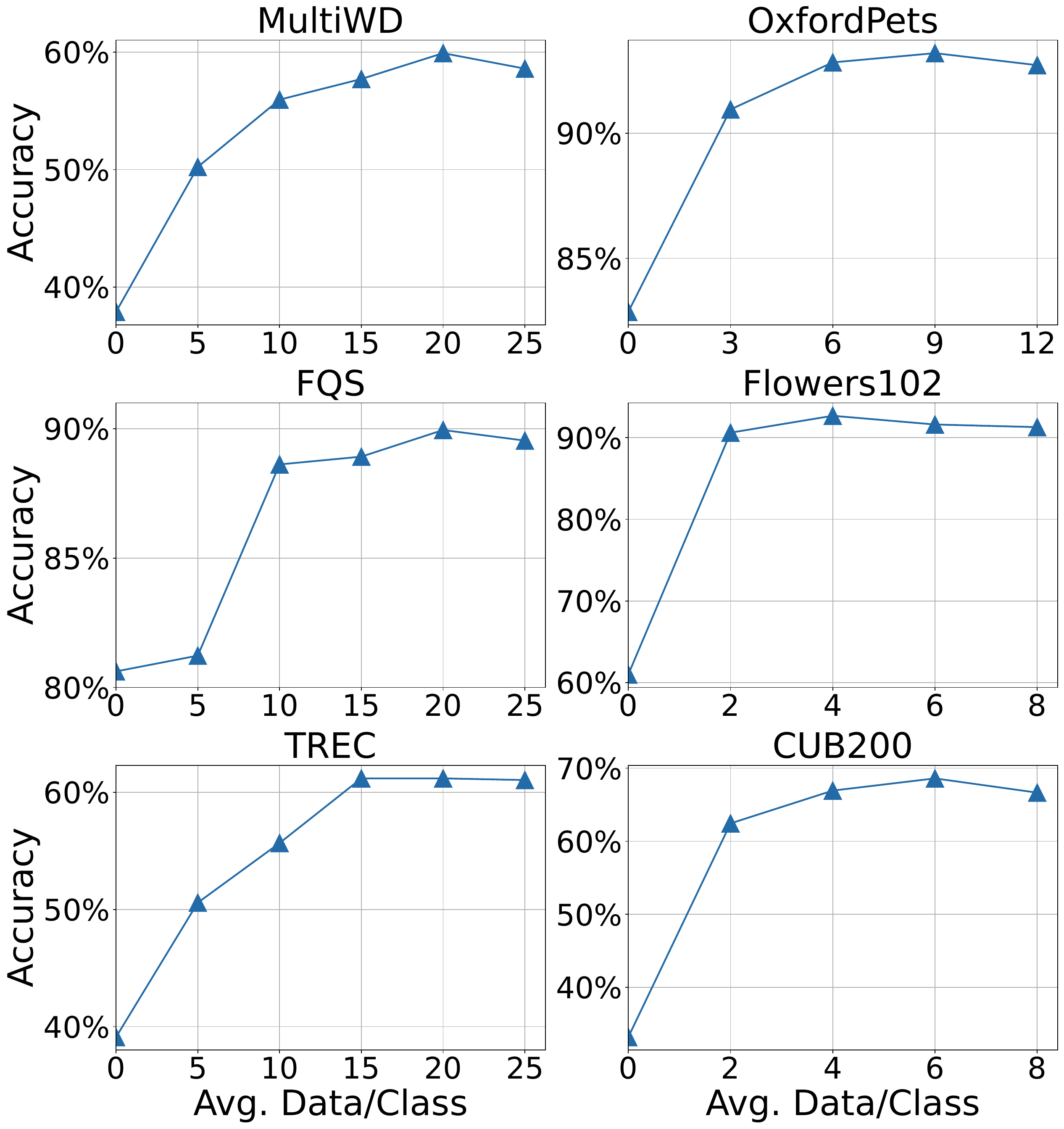}
    \vspace{-1em}
    \caption{Calibration Accuracy with respect to the average number of training samples per class. }
    \label{fig:exp_num_data}
\end{figure}

\paragraph{How does the exploration affect the training?} We also vary the exploration hyperparameter $\alpha$ to see how the exploitation-exploration trade-off in \design influences the calibration accuracy. As shown in \cref{fig:exp_alpha}, we can see that in all the cases, training with $\alpha>0$ will outperform training with $\alpha=0$. When $\alpha=0$, there is no exploration and the MAB will greedily select the class with the smallest empirical gradient shifting. However, the empirical gradient shifting $\hat\delta_t$ in few-shot training has a large bias from the true gradient shifting when estimating with a very limited number of samples. The large bias makes the greedy algorithm suffer from high regret when selecting according to the empirical gradient shifting, leading to poor convergence and low accuracy. On the other hand, when training with $\alpha>0$, we can notice that \design is not sensitive to the value of $\alpha$ and the accuracy increases slightly when using a larger $\alpha$. This is due to the small number of training samples in few-shot learning, making exploration more necessary to stabilize the convergence. This also justifies the rational of the relaxation of the confidence bound in \design, as we encourage more exploration even in the late stage of training.

\begin{figure}[t]
    \centering
    \includegraphics[width=1.0\linewidth]{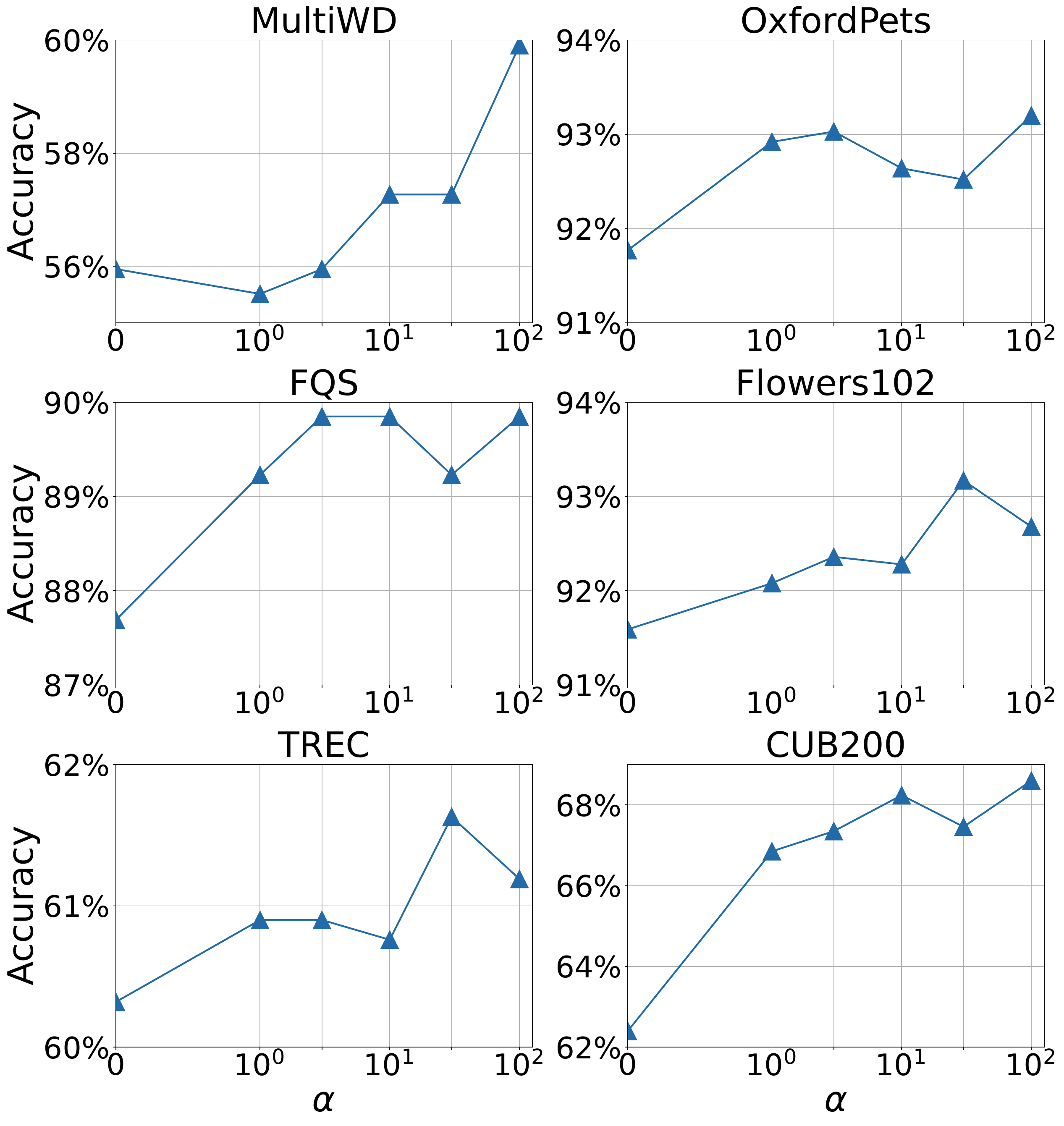}
    \vspace{-1em}
    \caption{Calibration Accuracy with respect to the exploration hyperparameter $\alpha$.}
    \label{fig:exp_alpha}
\end{figure}

\section{Conclusions}
 
\design offers a powerful and sample-efficient framework to enhance foundation models for recognizing implicit and abstract visual and textual patterns in specific tasks that require long-tail knowledge—a domain where pre-trained foundation models often fall short. By integrating a modified upper confidence bound algorithm into the adaptive data augmentation strategy, \design achieves optimal convergence with minimal computational overhead and limited training data. Its light-weight neural similarity network as the calibrator ensures high flexibility, strong adaptability, and low latency, making it highly suitable for resource-constrained environments. Both theoretical analysis and empirical assessments across multiple domains consistently demonstrate that \design achieves substantial accuracy improvements (up to 40\%) over state-of-the-art baselines, establishing a new paradigm for computation-and-data-efficient calibration for pattern recognition.

\clearpage
{
    \small
    \bibliographystyle{ieeenat_fullname}
    \bibliography{main}
}
\clearpage
\appendix
\section{Convergence Analysis}
\setcounter{theorem}{0}
\setcounter{assumption}{0}
\setcounter{lemma}{0}
\subsection{Convergence of Biased Gradient Descent}\label{subsec:conv}

Consider a generalized gradient descent training algorithm with the following update rule for each iteration:
\begin{equation}\label{eq:gd}
    \vw_{t+1} = \vw_t-\eta_t \vg_t, 
\end{equation}
where $\vw_t$ is the model parameter and $\eta_t$ is the learning rate at the $t$-th iteration. $\vg_t$ is the descent vector, while $\vg_t$ is not necessarily equal to the true gradient $\nabla L(\vw_t)$, with $L(\vw_t)=\mathbb{E}_{\rvx\sim p_\rvx}[l(\rvx;\vw_t)]$ as the task loss over the true data distribution $p_\rvx$. 

We make the following smoothness assumption, which is commonly used in the convergence analysis \citep{boyd2004convex}.
\begin{assumption}[$\beta$-smoothness] \label{ass:smooth}
$L(\vw_t)$ is differentiable with respect to $\vw_t$ and its gradient is $\beta$-Lipschitz: $\Vert\nabla L(\vw_a)-\nabla L(\vw_b)\Vert \le \beta \Vert \vw_a-\vw_b\Vert$. 
\end{assumption}

With this assumption, we can get the following theorem about the convergence of the generalized gradient descent.
\begin{theorem}\label{thm:conv}
    A gradient descent algorithm with \cref{eq:gd} as the update rule can achieve the following convergence rate with \cref{ass:smooth} and learning rate $\eta_t\le 1/\beta$:
    \begin{equation}\label{eq:conv}
        \inf_{t\le T}\Vert \nabla L(\vw_t)\Vert^2\le \frac{2L(\vw_1)}{\sum_{t=1}^T \eta_t}+\frac{\sum_{t=1}^T\eta_t\delta_t^2}{\sum_{t=1}^T \eta_t},
    \end{equation}
    where $\delta_t^2 = \Vert \vg_t-\nabla L(\vw_t)\Vert^2$ is the {gradient shifting}.
    \begin{proof}
        With Lipschitz smoothness in \cref{ass:smooth}, we have:
        \begin{align*}
            &L(\vw_{t+1})-L(\vw_t)\\
            \le &-\eta_t\langle\nabla L(\vw_t),\vg_t\rangle+\frac{\beta\eta_t^2}{2}\Vert \vg_t\Vert^2\\
            =&-\eta_t\langle\nabla L(\vw_t),\nabla L(\vw_t)+\vg_t-\nabla L(\vw_t)\rangle\\
            &+\frac{\beta\eta_t^2}{2}\Vert \nabla L(\vw_t)+\vg_t-\nabla L(\vw_t)\Vert^2\\
            =&(-\eta_t+\frac{\beta\eta_t^2}{2})\Vert\nabla L(\vw_t)\Vert^2+\frac{\beta\eta_t^2}{2}\Vert \vg_t-\nabla L(\vw_t)\Vert^2 \\
            &+ (-\eta_t+\beta\eta_t^2)\langle\nabla L(\vw_t),\vg_t-\nabla L(\vw_t)\rangle\\
            =&(-\eta_t+\frac{\beta\eta_t^2}{2})\Vert\nabla L(\vw_t)\Vert^2 \\
            &+ (-\eta_t+\beta\eta_t^2)\langle\nabla L(\vw_t),\vg_t-\nabla L(\vw_t)\rangle +\frac{\beta\eta_t^2}{2}\delta_t^2. 
        \end{align*}
        Since $\eta_t\le 1/\beta$, we have $-\eta_t+\beta\eta_t^2/2\le 0$ and $-\eta_t+\beta\eta_t^2\le 0$. We also have
        \begin{align*}
            &\Vert \vg_t-\nabla L(\vw_t)\Vert^2 = \delta_t^2\\
            \Rightarrow  &\langle\nabla L(\vw_t),\vg_t\rangle = \frac{\Vert \vg_t\Vert^2 + \Vert \nabla L(\vw_t)\Vert^2-\delta_t^2}{2}\\
            &\ge \frac{\Vert \nabla L(\vw_t)\Vert^2-\delta_t^2}{2}\\
            \Rightarrow &\langle\nabla L(\vw_t),\vg_t-\nabla L(\vw_t)\rangle = \langle\nabla L(\vw_t),\vg_t\rangle-\Vert \nabla L(\vw_t)\Vert^2\\
            &\ge -\frac{\Vert \nabla L(\vw_t)\Vert^2+\delta_t^2}{2}.
        \end{align*}
        And thus
        \begin{align*}
            &L(\vw_{t+1})-L(\vw_t) \\
            \le &(-\eta_t+\frac{\beta\eta_t^2}{2})\Vert\nabla L(\vw_t)\Vert^2\\
            &- \frac{-\eta_t+\beta\eta_t^2}{2}(\Vert \nabla L(\vw_t)\Vert^2+\delta_t^2) +\frac{\beta\eta_t^2}{2}\delta_t^2\\
            = & -\frac{\eta_t}{2}\Vert\nabla L(\vw_t)\Vert^2+\frac{\eta_t}{2}\delta_t^2.
        \end{align*}
        Rearanging the result, we get 
        \begin{align*}
            \eta_t\Vert\nabla L(\vw_t)\Vert^2\le 2(L(\vw_t)-L(\vw_{t+1}))+\eta_t\delta_t^2.
        \end{align*}
        Taking the sum over $t$, we get
        \begin{align*}
            \sum_{t=1}^{T}\eta_t\Vert\nabla L(\vw_t)\Vert^2\le 2(L(\vw_1)-L(\vw_{T+1}))+\sum_{t=1}^{T}\eta_t\delta_t^2.
        \end{align*}
        Without loss of generality, we assume $L(\vw)\ge 0$, and we divide both sides with $\sum \eta_t$ to get:
        \begin{align*}
            \frac{1}{\sum_{t=1}^{T}\eta_t}\sum_{t=1}^{T}\eta_t\Vert\nabla L(\vw_t)\Vert^2\le \frac{2L(\vw_1)}{\sum_{t=1}^{T}\eta_t}+\frac{\sum_{t=1}^{T}\eta_t\delta_t^2}{\sum_{t=1}^{T}\eta_t},
        \end{align*}
        which yields the final result:
        \begin{align*}
            \inf_{t\le T}\Vert \nabla L(\vw_t)\Vert^2\le \frac{2L(\vw_1)}{\sum_{t=1}^T \eta_t}+\frac{\sum_{t=1}^T\eta_t\delta_t^2}{\sum_{t=1}^T \eta_t}.
        \end{align*}
    \end{proof}
\end{theorem}

\subsection{Convergence Analysis of \design} \label{subsec:adamab_conv}
Without loss of Generality, we consider the case where we generate only one data in each iteration of gradient descent. And we omit the parameter $\vw_t$ in the following for simplicity. We use the following UCB algorithm to determine which class $C$ we should augment:
\begin{equation}\label{eq:lcb-da}
\begin{aligned}
    &C_t = \mathop{\arg\min}_C \quad \hat\delta_t^2(C)-\frac{\alpha}{\sqrt{n_{t-1}+1}}\sqrt{\frac{1}{n_{C,t-1}}},\\
    &\hat\delta_t^2(C)=\Vert \frac{1}{n_{t-1}+1}\nabla \hat L_C+\frac{n_{t-1}}{n_{t-1}+1} \nabla L(\sD_{t-1})-\nabla \hat L\Vert_2^2,\\
    &\nabla \hat L_C = \frac{1}{n_{C,t-1}}\sum_{\vx\in \sD_{C,t-1}}\nabla l(\vx),\\
    &\nabla \hat L = \frac{1}{K}\sum_{C=1}^K\nabla \hat L_C.
\end{aligned}
\end{equation}
We made the following assumption in addition to \cref{ass:smooth}:
\begin{assumption} [$\ell_\infty$-bounded gradients]\label{ass:grad_bound}
    The gradients for any data sample $\vx$ have $\Vert\nabla l(\vx)\Vert_\infty \le G$.
\end{assumption}
Before we prove the convergence result, we first prove the following lemma that will be useful in our analysis:
\begin{lemma}[Vector Hoeffding's Inequality]\label{lemma:class_conf}
    For a random vector $\rvx\in\mathbb R^{d}$ with $n$ I.I.D. data samples $\{\hat \vx_1,\hat \vx_2,\cdots,\hat \vx_n\}$, if $\Vert\rvx\Vert_\infty \le G$, for any positive constant $T$, we have 
    \begin{equation}
    \begin{aligned}
    &\pr{\norm{\frac{1}{n}\sum_{i=1}^n \hat\vx_i-\mathbb E[\rvx]}_2 \ge \epsilon}\le \frac{1}{T^4},\\
    \text{where }\quad &\epsilon = \sqrt{\frac{2dG^2\log(2T^4d)}{n}}.
    \end{aligned}
    \end{equation}
    \begin{proof}
        Since 
        \begin{align*}
        \norm{\frac{1}{n}\sum_{i=1}^n \hat\vx_i-\mathbb E[\rvx]}_2 \ge \epsilon\Rightarrow \norm{\frac{1}{n}\sum_{i=1}^n \hat\vx_i-\mathbb E[\rvx]}_\infty \ge \epsilon/\sqrt{d}
        \end{align*}
        We have 
        \begin{align*}
            &\pr{\norm{\frac{1}{n}\sum_{i=1}^n \hat\vx_i-\mathbb E[\rvx]}_2 \ge \epsilon}\\
            \le& \pr{\norm{\frac{1}{n}\sum_{i=1}^n \hat\vx_i-\mathbb E[\rvx]}_\infty \ge \epsilon/\sqrt{d}}\\
            =&\pr{\bigcup_{j=1}^d \left|\frac{1}{n}\sum_{i=1}^n \hat\vx_i^{(j)}-\mathbb E[\rvx^{(j)}]\right|\ge \epsilon/\sqrt{d}}\\
            \le& \sum_{j=1}^d \pr{\left|\frac{1}{n}\sum_{i=1}^n \hat\vx_i^{(j)}-\mathbb E[\rvx^{(j)}]\right|\ge \epsilon/\sqrt{d}}.
        \end{align*}
        With Hoeffding's Inequality, we have 
        \begin{align*}
            \Pr\left(\left|\frac{1}{n}\sum_{i=1}^n \hat\vx_i^{(j)}-\mathbb E[\rvx^{(j)}]\right|\ge \epsilon/\sqrt{d}\right)\le 2e^{-\frac{\epsilon^2n}{2G^2d}}\le \frac{1}{dT^4}.
        \end{align*}
        Thus, we get
        \begin{align*}
            \pr{\norm{\frac{1}{n}\sum_{i=1}^n \hat\vx_i-\mathbb E[\rvx]}_2 \ge \epsilon}\le \frac{1}{T^4}.
        \end{align*}
    \end{proof}
\end{lemma}
Now we can prove the following convergence result:
\begin{theorem}[Convergence of \design]
    Assuming \cref{ass:smooth} and \ref{ass:grad_bound}, the gradient descent with \design given in \cref{eq:lcb-da} can achieve the following convergence with a constant learning rate $\eta_t =\eta\le 1/\beta$ and some properly selected positive constant $\alpha$:
    \begin{equation}\label{eq:adamab_conv}
    \begin{aligned}
    &\inf_{t\le T}\mathbb E\Vert\nabla L(\vw_t)\Vert^2\\
    \le &\mathcal{O}\left(\frac{1}{T}\right)+\mathcal{O}\left(\sqrt{\frac{\log(T)}{T}}\right)+\sup_{t}\inf_{C}\delta_t^2(C).
    \end{aligned}
    \end{equation}
    \begin{proof}
        We first derive the confidence bound of estimating $\delta_t(C)$ for any class $C\in\{1,2,...,K\}$.
        We notice that
        \begin{align*}
        &|\hat \delta_t(C) -\delta_t(C)|\nonumber\\
        =&\Big|\Vert \frac{1}{n_{t-1}+1}\nabla \hat L_C+\frac{n_{t-1}}{n_{t-1}+1} \nabla L(\sD_{t-1})-\nabla \hat L\Vert_2\\
        &-\Vert \frac{1}{n_{t-1}+1}\nabla L_C+\frac{n_{t-1}}{n_{t-1}+1} \nabla L(\sD_{t-1})-\nabla L\Vert_2 \Big|\nonumber\\
        \le&\Vert \frac{1}{n_{t-1}+1}(\nabla \hat L_c-\nabla L_c)-(\nabla \hat L-\nabla L)\Vert_2\\
        \le &\frac{1}{n_{t-1}+1}\Vert \nabla \hat L_c-\nabla L_c\Vert_2+\Vert \nabla \hat L-\nabla L\Vert_2.
        \end{align*}
        And thus we have
        \begin{align}
            &\Pr(\Vert \nabla \hat L_c-\nabla L_c\Vert_2\le (1+n_{t-1})\epsilon_1(C)\nonumber\\
            &\qquad\cap\Vert \nabla \hat L-\nabla L\Vert_2\le \epsilon_2)\nonumber\\
            &\le \Pr(|\hat \delta_t(C) -\delta_t(C)|\nonumber\le \epsilon_1(C)+\epsilon_2)\nonumber\\
            \Rightarrow &\Pr(|\hat \delta_t(C) -\delta_t(C)|\ge \epsilon_1(C)+\epsilon_2)\nonumber\\
            &\le \Pr(\Vert \nabla \hat L_c-\nabla L_c\Vert_2\ge (1+n_{t-1})\epsilon_1(C)\nonumber\\
            &\qquad\cup\Vert \nabla \hat L-\nabla L\Vert_2\ge \epsilon_2)\nonumber\\
            &\le \Pr(\Vert \nabla \hat L_c-\nabla L_c\Vert_2\ge (1+n_{t-1})\epsilon_1(C))\nonumber\\
            &\quad+\Pr(\Vert \nabla \hat L-\nabla L\Vert_2\ge \epsilon_2)\label{eq:bound_all}
        \end{align}
        For the first term in \cref{eq:bound_all}, with \cref{lemma:class_conf}, we have:
        \begin{equation}\label{eq:bound_t1}
            \begin{aligned}
            &\Pr(\Vert \nabla \hat L_c-\nabla L_c\Vert_2\ge (n_{t-1}+1)\epsilon_1(C))\\
            \le&\Pr(\Vert \nabla \hat L_c-\nabla L_c\Vert_2\ge \sqrt{n_{t-1}+1}\epsilon_1(C))\le \frac{1}{T^4},\\
            \text{where }\quad&\epsilon_1(C) = \frac{1}{\sqrt{n_{t-1}+1}}\sqrt{\frac{2dG^2\log(2T^4d)}{n_{C,t-1}}}.
            \end{aligned} 
        \end{equation}
    And for the second term, we have
    \begin{align*}
        \Vert \nabla \hat L-\nabla L\Vert_2 = \Vert \frac{1}{K}\sum_{C=1}^K (\nabla \hat L_C - \nabla L_C)\Vert_2\\
        \le \frac{1}{K}\sum_{C=1}^K \Vert \nabla \hat L_C - \nabla L_C\Vert_2.
    \end{align*}
    And thus,
    \begin{align*}
        &\forall C, \Vert \nabla \hat L_C-\nabla L_C\Vert_2\le \epsilon_C \\
        \Rightarrow &\Vert \nabla \hat L-\nabla L\Vert_2\le \frac{1}{K}\sum_{C=1}^K \epsilon_C.
    \end{align*}
    This leads to
    \begin{align*}
        &\Pr(\bigcap_{C=1}^K \Vert \nabla \hat L_C-\nabla L_C\Vert_2\le \epsilon_C)\\
        &\le \Pr(\Vert \nabla \hat L - \nabla L\Vert_2\le \frac{1}{K}\sum_{C=1}^K\epsilon_C)\\
        \Rightarrow &\Pr(\Vert \nabla \hat L - \nabla L\Vert_2\ge \frac{1}{K}\sum_{C=1}^K\epsilon_C)\\
        &\le \Pr(\bigcup_{C=1}^K \Vert \nabla \hat L_C-\nabla L_C\Vert_2\ge \epsilon_C)\\
        &\le \sum_{C=1}^K\Pr(\Vert \nabla \hat L_C-\nabla L_C\Vert_2\ge \epsilon_C).
    \end{align*}
    With \cref{lemma:class_conf}, we have
    \begin{equation*}
    \begin{aligned}
        &\Pr(\Vert \nabla \hat L_C-\nabla L_C\Vert_2 \ge \epsilon_C)\le \frac{1}{KT^4},\\
        \text{with }\quad & \epsilon_C = \sqrt{\frac{2dG^2\log(2dKT^4)}{n_C}}.
    \end{aligned}
    \end{equation*}
    Thus if we choose
    \begin{equation*}
        \epsilon_2 = \frac{1}{K}\sum_{C=1}^K\sqrt{\frac{2dG^2\log(2dKT^4)}{n_C}},
    \end{equation*}
    we achieve
    \begin{equation}\label{eq:bound_t2}
        \Pr(\Vert \nabla \hat L-\nabla L\Vert_2\ge \epsilon_2)\le \frac{1}{T^4}.
    \end{equation}
    Plugging \cref{eq:bound_t1} and \cref{eq:bound_t2} into \cref{eq:bound_all}, we get the confidence bound for $\hat\delta_t$:
    \begin{equation*}
        \Pr(|\hat \delta_t(C) -\delta_t(C)|\ge \epsilon_1(C)+\epsilon_2)\le \frac{2}{T^4}
    \end{equation*}
    with 
    \begin{equation}\label{eq:eps}
    \begin{aligned}
        &\epsilon_1(C) = \frac{1}{\sqrt{n_{t-1}+1}}\sqrt{\frac{2dG^2\log(2T^4d)}{n_{C,t-1}}},\\
        &\epsilon_2 = \frac{1}{K}\sum_{C=1}^K\sqrt{\frac{2dG^2\log(2dKT^4)}{n_{C,t-1}}}.
    \end{aligned}
    \end{equation}
    We can easily extend this result to get the confidence bound of $\hat\delta_t^2$, since
    \begin{align*}
        |\hat\delta_t^2-\delta_t^2|=(\hat\delta_t+\delta_t)|\hat\delta_t-\delta_t|.
    \end{align*}
    And we also have
    \begin{align}
        &\hat\delta_t\nonumber\\ 
        = &\Vert \frac{1}{n_{t-1}+1}\nabla \hat L_C+\frac{n_{t-1}}{n_{t-1}+1} \nabla L(\sD_{t-1})-\nabla \hat L\Vert_2\nonumber\\
        \le  &\Vert \frac{1}{n_{t-1}+1}\nabla \hat L_C\Vert_2+\Vert\frac{n_{t-1}}{n_{t-1}+1} \nabla L(\sD_{t-1})\Vert_2+\Vert \nabla \hat L\Vert_2\nonumber\\
        \le  &2\sqrt{d}G,\label{eq:delta_bound}
    \end{align}
    and similarly for $\delta_t$. Accordingly, we have 
    \begin{align*}
        &\Pr(|\hat\delta_t^2-\delta_t^2|\le 4\sqrt{d}G(\epsilon_1(C)+\epsilon_2))\\
        \ge &\Pr(|\hat\delta_t-\delta_t|\le \epsilon_1(C)+\epsilon_2))\ge 1-\frac{2}{T^4}.
    \end{align*}
    With the confidence bound, we can prove the instantaneous regret at each iteration as follows. Assuming that the optimal class selection is
    \begin{align*}
        C^* = \mathop{\arg\min}_C \delta_t^2(C).
    \end{align*}
    Let $\epsilon_t(C) = 4\sqrt{d}G(\epsilon_1(C)+\epsilon_2)$. Notice that in \cref{eq:lcb-da}, the selection algorithm is the same as we select with $\hat \delta_t^2(C)-\epsilon_t(C)$ as the objective function when $\alpha = 4dG^2\sqrt{2\log(2T^4d)}$. We select $C$ instead of $C^*$ meaning that:
    \begin{align*}
        \hat \delta_t^2(C)-\epsilon_t(C)\le \hat \delta_t^2(C^*)-\epsilon_t(C^*)\\
        \Rightarrow \hat \delta_t^2(C)-\hat \delta_t^2(C^*)\le \epsilon_t(C)-\epsilon_t(C^*).
    \end{align*}
    Accordingly, we have the following instantaneous regret:
    \begin{align}
        &\mathbb E[\delta_t^2(C)-\delta_t^2(C^*)]\nonumber\\
        =&\mathbb E[\delta_t^2(C)-\hat\delta_t^2(C)+\hat\delta_t^2(C)-\hat\delta_t^2(C^*)+\hat\delta_t^2(C^*)-\delta_t^2(C^*)]\nonumber\\
        \le &\epsilon_t(C)+\epsilon_t(C)-\epsilon_t(C^*)+\epsilon_t(C^*)\nonumber\\
        =& 2\epsilon_t(C).\label{eq:inst_reg}
    \end{align}
    We assert that $\forall C,C' \in \{1,2,...,K\}$ and $\forall t\le T$, there exists a positive constant $M$ such that  $n_{C,t-1}/n_{C',t-1}\le M$. Assume that there exist two classes $C$ and $C'$ with $n_{C,t-1}=Mn_{C',t-1}$. To select $C$ instead of $C'$ according to \cref{eq:lcb-da}, we must have 
    \begin{align}
        &\hat \delta_t^2(C)-\frac{\alpha}{\sqrt{n_{t-1}+1}}\sqrt{\frac{1}{n_{C,t-1}}}\nonumber\\
        &\le \hat \delta_t^2(C')-\frac{\alpha}{\sqrt{n_{t-1}+1}}\sqrt{\frac{1}{n_{C',t-1}}}\nonumber\\
        \Rightarrow &\hat\delta_t^2(C')-\hat\delta_t^2(C)\nonumber\\
        &\ge \frac{\alpha}{\sqrt{n_{t-1}+1}}\sqrt{\frac{1}{n_{C',t-1}}}-\frac{\alpha}{\sqrt{n_{t-1}+1}}\sqrt{\frac{1}{n_{C,t-1}}}\nonumber\\
        &=\frac{\alpha}{\sqrt{(n_{t-1}+1)n_{C',t-1}}}(1-\sqrt{\frac{1}{M}})\nonumber\\
        &>\frac{\alpha \sqrt{M+1}}{n_{t-1}+1}(1-\sqrt{\frac{1}{M}})\label{eq:ratio_bound}
    \end{align}
    \cref{eq:ratio_bound} comes from the fact that $n_{C,t-1}+n_{C',t-1}=(1+M)n_{C',t-1}\le n_{t-1}< n_{t-1}+1$. For the left-hand side, we have 
    \begin{align}
        &\hat\delta_t^2(C')-\hat\delta_t^2(C)\nonumber\\
        \le&(\hat\delta_t(C')+\hat\delta_t(C))|\hat\delta_t(C')-\hat\delta_t(C)|\nonumber\\
        \le& \frac{4\sqrt{d}G}{n_{t-1}+1}\Vert\nabla \hat L_C-\nabla \hat L_{C'}\Vert\label{eq:diff_bound}\\
        \le& \frac{8dG^2}{n_{t-1}+1}\nonumber
    \end{align}
    where \cref{eq:diff_bound} comes from \cref{eq:delta_bound} and triangle inequality, while the last inequality comes from \cref{ass:grad_bound}. Accordingly, if $M$ satisfies
    \begin{align*}
        8dG^2 \le \alpha \sqrt{M+1}(1-\sqrt{\frac{1}{M}}),
    \end{align*}
    we can guarantee that $n_{C,t-1}/n_{C',t-1}\le M$. We can simplify this condition as
    \begin{align}\label{eq:bound_M}
        M\ge\left(1+\sqrt{\frac{2}{\log(2T^4d)}}\right)^2.
    \end{align}
    We further set $M=(1+\sqrt{2/\log(2d)})^2$ to get rid of $T$, then we can get that
    \begin{align}
        \sum_{C=1}^K n_{C,t-1}=t-1+n_0\Rightarrow \forall C, n_{C,t-1}\ge \frac{t-1+n_0}{M(K-1)+1}.\label{eq:num_sel_bound}
    \end{align}
    With \cref{eq:eps}, \cref{eq:inst_reg} and \cref{eq:num_sel_bound}, we get the instantaneous regret as
    \begin{equation}
        \mathbb E[\delta_t^2(C)-\delta_t^2(C^*)]\le 2\epsilon_t(C)\le \mathcal{O}(\sqrt{\frac{\log(T)}{t}}),\label{eq:instant_regret_final}
    \end{equation}
    Taking Expectation over the result of \cref{thm:conv} and plugging \cref{eq:instant_regret_final} in, we complete the proof with the fact that
    \begin{equation*}
        \frac{\sum_{t=1}^T \eta_t\delta_t^2(C_t^*)}{\sum_{t=1}^T \eta_t}\le \sup_t\delta_t^2(C_t^*)=\sup_t\inf_C\delta_t^2(C).
    \end{equation*}
    \end{proof}
\end{theorem}

\begin{remark}\label{remark:relax_cb}
    It is noteworthy that we relax the confidence bound $\epsilon_1(C)$ in \cref{eq:bound_t1}, whose tight confidence bound should be 
    \begin{equation*}
        \epsilon_1^*(C) = \frac{1}{n_{t-1}+1}\sqrt{\frac{2dG^2\log(2T^4d)}{n_{C,t-1}}}.
    \end{equation*}
    We notice that the tight confidence bound $\epsilon_1^*(C)$ decreases so fast that some classes cannot be guaranteed to be sufficiently explored to ensure a constant $M$ as given in \cref{eq:bound_M}. In fact, the $M$ needs to be proportional to $t$ with $\epsilon_1^*(C)$, which makes $\epsilon_t(C)$ not able to decrease along $t$ as in \cref{eq:instant_regret_final}, and thus the convergence cannot be achieved. 
\end{remark}
\section{Experimental Settings}\label{apx:exp}

All the experiments are completed on a single MacBook Pro with a single M4 Max chip and 36GB memory. 

\subsection{Training Hyperparameters}

In all the calibration methods, we adopt an Adam Optimizer \cite{kingma2014adam} with initial learning rate $\eta_0=0.005$ and a weight decay $\lambda=0.0001$. We use a Cosine Annealing learning rate scheduler with $\eta_T=0.5\eta_0$. For the other hyperparameters, we summarize them in \cref{tab:hyperparameters}. ``\# Training Rounds'' means the total number of training epochs we conducted, and ``\# Aug. Rounds'' denotes the number of augmentation epochs, i.e., we only generate data in the first ``\# Aug. Rounds'', and then we do not generate but only train the model until ``\# Training Rounds''.

We conduct experiments on a wide range of hyperparaters and select the best hyperparameters with the highest accuracy. Specifically, we tried initial learning rate from $\{0.01,0.005,0.001,0.0005\}$ and batch size from $\{16,32,64,128,256,512\}$. We also tried different $\alpha$ from $\{0,1,3,10,30,100\}$ and different number of augmentation rounds from $\{0,K,2K,3K\}$ where $K$ is the number of classes. We report the results with different $\alpha$ and number of augmentation rounds in \cref{subsec:ablation}.

\begin{table}[t]
\centering
\caption{The hyperparameters for the results in \cref{tab:main_res_text} and \cref{tab:main_res_img}.}
\label{tab:hyperparameters}
\resizebox{\linewidth}{!}{
\begin{tabular}{c|cccccc}
\hline
Dataset            & MultiWD & FQS & TREC & OxfordPets & Flowers102 & CUB200 \\ \hline
Batch Size         & 32      & 32  & 64   & 64           & 128        & 256    \\
\# Training Rounds & 30      & 50  & 100  & 100          & 300        & 500    \\
$\Delta n$      & 5       & 5   & 5    & 3            & 2          & 2      \\
\# Aug. Rounds     & 18      & 39  & 60   & 74           & 102        & 400    \\
$\alpha$              & 100     & 100 & 30   & 100          & 30         & 100    \\ \hline
\end{tabular}
}
\end{table} 

\subsection{Prompt Templates}

We use one text dataset and one image dataset as examples to demonstrate how to construct prompt templates for in-context learning, reranking, embedding, and data augmentation, respectively. Notice that for the few-shot setting, we will add all the examples in the initial training set to the description of each label, while no examples will be included for the zero-shot setting.

\noindent\textbf{Forbidden Question Set (FQS):}

\begin{lstlisting}[language=yaml,caption={FQS prompts.}]
in_context_learning_prompt: |
    You are a Responsible AI system checking if a user query complies with certain safety guidelines and application-specific rules. Each rule will be provided with a detailed description and some example queries that violate the rule. You need to judge whether a query provided by the user violates any of the rules and which one it is most likely to violate.
  
    Given the following rules with their corresponding indices:
    
    1. {label1}: {description1}
    2. {label2}: {description2}
    ...
    K. {labelK}: {descriptionK}
    
    Evaluate the following user query carefully against the rules: 
    
    User Query: {query}
    
    Please make a judgment which rule this query violates. Output your judgment as a single integer: violated_rule_index. Do not include anything else except for violated_rule_index in your response.
    If none of the rules are violated, return -1 as the violated_rule_index.

reranker_query_prompt: |
    Please find the rule that the following query violates:
    {query}

reranker_document_prompt: |
    Rule: 
    {label}: {description}

embedding_query_prompt: |
    {query}

embedding_label_prompt: |
    Rule: 
    {label}: {description}

augmentation_prompt: |
    Here is a rule that defines the topics of the queries that should not be answered by an LLM agent: 
    {label}: {description}
    
    Please generate {num_generate} **diverse** example queries that violate this rule. 
    The examples should be as **diverse** as possible so that the agent can learn from the examples if a query violates the rule. In addition, do not generate examples that are similar to the following existing examples:
    {existing_queries}
    
    In your response, only list the examples with one example in each line, without any numbering or bullets.
\end{lstlisting}

\noindent\textbf{OxfordIIITPet (OxfordPets):}

\begin{lstlisting}[language=yaml,caption={OxfordPets prompts.}]
in_context_learning_prompt: |
    You will be given an image and asked to identify the species of the pet in the image. A predefined set of pet species will be provided, and you need to select one species from them for the pet in the given image.
    
    Given the following pet species with their corresponding indices:
    
    1. {label1}
    2. {label2}
    ...
    K. {labelK}
    
    Check the image carefully to determine which species the pet belongs to.

    Output your decision as a single integer: species_index. Do not include anything else except for species_index in your response.
    If none of the species apply, return -1 as the species_index.

reranker_query_prompt: |
    {query}
    
reranker_document_prompt: |
    A picture of a pet belonging to {label} species.

embedding_query_prompt: |
    {query}
    
embedding_label_prompt: |
    A picture of a pet belonging to {label} species.
  
augmentation_prompt: |
    Generate pictures of a pet belonging to {label} species in different scenarios.
\end{lstlisting}

\end{document}